\documentclass{new_tlp}
\usepackage{mathptmx}
\usepackage{amsmath}
\usepackage{graphicx}
\usepackage{multirow}
\usepackage{natbib}
\usepackage{times}
\renewcommand*\ttdefault{txtt}
\usepackage{soul}
\usepackage{url}
\usepackage{amssymb}
\usepackage[ruled,vlined]{algorithm2e}
\usepackage{courier,helvet,times}
\newcommand\bcmdtab{\noindent\bgroup\tabcolsep=0pt%
  \begin{tabular}{@{}p{10pc}@{}p{20pc}@{}}}
\newcommand\ecmdtab{\end{tabular}\egroup}

  \title[Theory and Practice of Logic Programming]
        {Learnability with PAC Semantics for \\   Multi-agent Beliefs}

  \author[]{IONELA G. MOCANU\\
         The University of Edinburgh\\
         \email{i.g.mocanu@ed.ac.uk}\\\\
         \normalsize \textnormal{VAISHAK BELLE}\\
         The University of Edinburgh\\
         \email{vbelle@ed.ac.uk} \\\\
         \normalsize \textnormal{BRENDAN JUBA} \\
         Washington University in St. Louis\\
         \email{bjuba@wustl.edu}
         }

\jdate{May 2023}
\pubyear{2023}
\pagerange{\pageref{firstpage}--\pageref{lastpage}}

\newtheorem{THEOREM}{Theorem}
                        {\end{THEOREM}}

\newtheorem{LEMMA}[THEOREM]{Lemma}
                      {\end{LEMMA}}
\newtheorem{COROLLARY}[THEOREM]{Corollary}
                          {\end{COROLLARY}}
\newtheorem{PROPOSITION}[THEOREM]{Proposition}
                            {\end{PROPOSITION}}
                            
\newtheorem{DEFINITION}[THEOREM]{Definition}
                            {\end{DEFINITION}}
\newtheorem{CLAIM}[THEOREM]{Claim}
                            {\end{CLAIM}}
\newtheorem{EXAMPLE}[THEOREM]{Example}
\newtheorem{REMARK}[THEOREM]{Remark}
                            {\end{REMARK}}
                            
\newtheorem{NOTATION}[THEOREM]{Notation}
							                            {\end{NOTATION}}

\begin{document}

\label{firstpage}

\maketitle

  \begin{abstract}
The tension between deduction and induction is perhaps the most fundamental issue in areas such as philosophy, cognition and artificial intelligence. In an influential paper, {\it Valiant} recognized that the challenge of learning should be integrated with deduction. In particular, he proposed a semantics to capture the quality possessed by the output of {\it Probably Approximately Correct} (PAC) learning algorithms when formulated in a logic. Although weaker than classical entailment, it allows for a powerful model-theoretic framework for answering queries. In this paper, we provide a new technical foundation to demonstrate PAC  learning with multiagent epistemic logics. To circumvent the negative results in the literature on the difficulty of robust learning with the PAC semantics, we consider so-called implicit learning where we are able to incorporate observations to the background theory in service of deciding the entailment  of an epistemic query. We prove correctness of the learning procedure and discuss results on the sample complexity, that is how many observations we will need to provably assert that the query is entailed  given a user-specified error bound. Finally, we investigate under what circumstances this algorithm can be made efficient. On the last point, given that reasoning in epistemic logics especially in multi-agent epistemic logics is {PSPACE}-complete, it might seem like there is no hope for this problem. We leverage some recent results on the so-called {\it Representation Theorem} explored for single-agent and multi-agent epistemic logics with the \textit{only knowing} operator to reduce modal reasoning to propositional reasoning.  
  \end{abstract}

  \begin{keywords}
 Multi-Agent Systems, Knowledge Acquisition, Only-Knowing, Efficient Reasoning.
  \end{keywords}

\tableofcontents

\section{Introduction}

An increasing number of agent-based technologies, which involve automated reasoning, such as self-driving cars or house robots are widely deployed. In particular, many \textit{AI} applications model environments with multiple agents, where each agent acts using their own knowledge and beliefs to achieve goals either by coordinating with the other agents or by challenging an opponent's actions  in a competitive context. Reasoning not just about the agent's world knowledge but also about other agents' mental state is referred to as \textit{epistemic reasoning}, for which a variety of modal logics have been developed \citep{reasoning:about:knowledge}. Epistemic modal  logic is widely recognised as a specification language for a range of domains, including  robotics, games, and  air traffic control \citep{belardinelli2007}.

While a number of sophisticated formal logics have been proposed for modelling such contexts, from areas such as philosophy, knowledge representation and game theory, they do not, to a large extent, address the problem of \textit{knowledge acquisition}. Classically, given a set of observations, the most common approach is that of explicit hypothesis construction, as seen in \textit{inductive logic programming} \citep{Muggleton1994} and \textit{statistical relational learning} \citep{getoor2007introduction,de2011statistical}. Here, we construct sentences in the logic that either entail observations or capture associations in those with high probability. By contrast, a recent line of work initiated the idea of an \textit{implicit knowledge base} constructed from observations \citep{juba2013}. The implicit approach avoids the construction of an explicit hypothesis but still allows us to reason about queries against noisy observations. This is motivated by tractability: in agnostic learning \citep{kearns1994toward}, for example, where one does not require examples (drawn from an arbitrary distribution) to be fully consistent with learned sentences, efficient algorithms for learning conjunctions in propositional logic would yield an efficient algorithm for PAC-learning DNF (also over arbitrary distributions), which current evidence suggests to be intractable \citep{daniely2016complexity}. Since the discovery of this technique, learning with the PAC-semantics has been extended to certain
fragments of first-order logic   \citep{bellejuba2019}. Given the promise of this technique, but also taking into consideration the hardness of reasoning in epistemic logic (PSPACE-complete when there is more than one agent \citep{reasoning:about:knowledge}), we continue this line of work for the problem of implicitly learning with epistemic logic. 

The extension to epistemic logics raises numerous challenges not previously considered by any other work on the PAC-semantics. In the first place, we must describe the learning process in a multi-agent epistemic framework, where previously the PAC-semantics had only been considered as an extension of Tarskian semantics. In addition, implicit learning generally relies on three steps: first, to argue that the way of accepting the observations with background theory and accepting a high number of them is correct as per PAC-semantics. Secondly, to measure the sample complexity, that is, how many observations are required to provably assert that the query is entailed given a user-specified error bound. Finally, we want to look at under what circumstances could this algorithm achieve a polynomial run time. On the last point, given that reasoning in epistemic logics, especially in multi-agent epistemic logics is PSPACE-complete \citep{halpern1997critical, Bacchus1999171, 174658, halpern2003uncertainty} it might seem like there is no hope for this problem. In this article, we provide in fact, concrete results about sample complexity and correctness, as well as polynomial time guarantees under certain assumptions. Our learning task is similar to an unsupervised learning model, however, our end task is deciding query entailment with respect to background knowledge and partial interpretations.

In this work, we show how to extend the implicit learning approach to epistemic modal formulas, yielding agnostic (implicit) learning of epistemic formulas for the purposes of deciding entailment queries. We leverage some recent results on the so-called \textit{Representation Theorem} explored for single-agent and multi-agent epistemic logics \citep{Levesque2001logic, bellejair, Schwering19}. In these results, in addition to the standard operator for knowledge, a modal operator for \textit{only knowing} is introduced \citep{LEVESQUE1990alliknow}, which provides a means to succinctly characterise all the beliefs as well as the non-beliefs of the agent. For example, only knowing proposition $p$,  denoted as $O(p)$, entails knowing $K(p)$: $O(p) \vDash K(p)$, however, only knowing $p$ does not entail another proposition $q$: $O(p) \nvDash q $, for all $p \neq q$. Thus, this is quite attractive to capture everything that is known. It can be shown that to check the validity of $O (\phi) \rightarrow K (\alpha)$ when $\phi$ is objective and $\alpha$ can mention any number of $K_i$ modalities in the presence of negation, conjunction, and disjunction, can be reduced to propositional reasoning. Although propositional reasoning is already NP-complete, it is known that there are a number of approaches for tractability including bounded space treelike resolution \citep{esteban2001} and bounded-width resolution \citep{galil1977}. In the multi-agent setting, the natural extension of this to the Representation Theorem is allowing for a knowledge base of this sort $O_A(\phi \wedge O_B (\psi \wedge ...) \wedge O_C(...))$, which specifies everything that the root agent, say $A$, knows as well as everything that the root agent believes agent $B$ knows and $C$ knows and so on. This admittedly can seem like a strong setting but recent results have shown how this can be relaxed \citep{Schwering19}. The other way to motivate this approach is that initially, perhaps nothing is known or all start with common beliefs, and then new knowledge can be acquired as actions happen. For example, in a paper by Belle and Lakemeyer \citep{Belle:2015ac}, it is shown how the setting provides a natural way to capture the muddy children puzzle \citep{reasoning:about:knowledge}. Under the assumption that you have one of these background theories, and we are interested in the entailment  of $K_A \alpha$, where $\alpha$ can mention any number of $K_i$ operators for any $i$ and arbitrarily nesting, the Representation Theorem establishes that this reduces to propositional reasoning. However, even though this holds, the key concern is because we have specified what is only known, we need a way to incorporate observations to formalize learning  from such observations. We focus on the case where an agent in the system wishes to use learning to update its knowledge base. We allow for a new modality, an observational modality $[\rho]$ that we borrow from multi-agent dynamic logics with a regression operator \citep{bellejair}, and show that this provides a logically correct approach for incorporating observations and thereby checking the entailment of the query. Beyond the novelty of our technical results, we also note that there are very few approaches for knowledge acquisition with epistemic logics despite them being one of the most popular modal logics in the knowledge representation community. That is, with this paper, we are making advances both in machine learning as well as the knowledge representation literature.

\section{Preliminaries}

We define the reasoning problem as follows: in a system of multiple agents, each agent has some background knowledge encoded in a knowledge base and receives information about the environment through the sensors, which are encoded as partial observations. We then ask the root agent queries about the environment. After receiving the partial observations, the agent returns with some degree of validity the answer for the specified query. In epistemic reasoning, we distinguish between what is true in the real world and what the agents know or believe about the world. For example, the beliefs of agent $A$ about the world may differ from agent $B$’s knowledge, and what agent $A$ believes $B$ to know may differ from what $B$ actually believes. In the context of multi-agent reasoning, we are interested in deciding the entailment of an input query about the other agents with respect to a background theory which contains the beliefs of agents in the application domain.

\noindent\textbf{Syntax.}
Let $\mathcal{L}_n$ be a propositional  language which consists of formulas from the finite set $P$ of propositions 
and connectives $ \wedge, \vee, \neg, \rightarrow$. Let $\mathcal{OL}_n$ be the epistemic language with additional modal operators. First,  $K_i$:  $K_i \alpha$ is to be read as ``agent $i$ knows $\alpha$'', where $i$ ranges over the finite set of agents $Ag = \{ A, B\}$, which, for simplicity, assumes two agents, although this can be extended to many more agents. Second, $O_i \alpha$ is to be read as \textit{``all that agent i knows is $\alpha$''} to express that a formula is \textit{all} that is known. The \textit{only knowing} operator is instrumental in capturing the beliefs as well as the non-beliefs of agents \citep{bellejair, LEVESQUE1990alliknow}.

Somewhat unusually, as discussed above, borrowing from the dynamic version of $\mathcal{OL}_n$  \citep{bellejair}, we introduce a dynamic operator $[\rho]$ such that $[\rho] \alpha$ is understood as formula $\alpha$ is \textit{true} after receiving the observation $\rho$. In particular, assume a finite set of observations $OBS$ with elements consisting of conjunctions over the set $P$, e.g., $OBS = \{ p, p \land q, \ldots, (p \land \neg q) \land r \}$. The elements of $OBS$ are used strictly within the dynamic operator $[\rho]$. In order to interpret the action symbol, we will introduce a sensing function, one corresponding to each agent $obs_i$, that takes as argument the action symbol and returns either the observation it corresponds to or simply returns $true$. A well-defined formula is then of the form $[\rho] \alpha$ where $\alpha$ is either propositional or at most mentions knowledge modalities $K_i$. For example $obs_A(p \wedge q) = p \wedge q$.
It is not necessary that $obs_A(p \wedge q) = obs_B(p \wedge q)$. For instance, we may also have $obs_B(p\land q) = true$ and $obs_A(p \land q)  = p$. In other words, the agents may obtain different observations from the same observational action. Suppose agent $B$ looks at a card, he will be able to read what is written on the card whereas every other agent now knows that $B$ has read the card but not what the card says. This is a simplified account from previous work by Belle and Lakemeyer \citep{bellejair}, mainly because we need to deal with a single observation for the purposes of this paper. It will be straightforward to extend it to a sequence of observations, however. Moreover, we will be appealing to \textit{Regression} over sensing actions from that work in our approach.

\noindent\textbf{Semantics.} The semantics is provided in terms of possible worlds and \textit{k-structures} \citep{bellejair}. We distinguish the mental state of an agent from the real world and make use of epistemic states to model different mental possibilities. The standard literature uses the \textit{Kripke} \textit{structure} to model multi-agent epistemic states \citep{reasoning:about:knowledge}.  For this work, we use \textit{k-structures} \citep{bellejair} instead, which deviates from the Kripke structure in the way the epistemic state is defined.
The $k$-structure uses sets of worlds at different levels, the idea being that the number of levels corresponds to the number of alternating modalities in the formula. The $i$-depth is defined as follows:

\begin{DEFINITION} [$i$-depth \citep{bellejair}] The $i$-depth of $\alpha \in \mathcal{O L}_n$ where $i$ is the agent's index, denoted $|\alpha| _i $, is defined inductively as:

\begin{itemize}
    \item[] \hspace{0.5cm} $\bullet$  $|\alpha|_i = 1$ for propositions;
    \item[] \hspace{0.5cm} $\bullet$  $|\neg \alpha |_i= |\alpha|_i$;
    \item[] \hspace{0.5cm} $\bullet$  $|\alpha \vee \beta| _i= max (|\alpha|_i, |\beta|_i)$;
    \item[] \hspace{0.5cm} $\bullet$ $|[\rho] \alpha|_i = |\alpha|_i$, where $\rho \in OBS$;
    \item[] \hspace{0.5cm} $\bullet$  $|M_i \alpha| _i = |\alpha|_i, M_i \in \{ K_i, O_i\}$, and
    \item[] \hspace{0.5cm} $\bullet$  $|M_j \alpha |_i = |\alpha|_j +1, M_j \in \{K_j, O_j \}$ and $i \neq j$.

\end{itemize}

\end{DEFINITION}

Note that the dynamic modality has no impact on the $i$-depth of the formula, as it does not refer to the agent's knowledge. The reason for choosing  $k$-structures instead of the classical Kripke is because it provides a very simple semantics for only knowing in the multi-agent case \citep{DBLP:conf/dagstuhl/Belle10}. Beliefs are reasoned about in terms of valid sentences of the form: $O_A(\Sigma) \vDash K_A \alpha$, read as ``if $\Sigma$ is all that the agent $A$ believes, then the agent knows $\alpha$''. In the interest of simplicity, we focus for the rest of the paper on only two agents $A$ and $B$, where moreover $A$ is the root agent, in the sense that we will be interested in what $A$ knows and what $A$ observes, and the queries will be posed regarding $A$'s knowledge. The $i$-depth of a formula  $\alpha$ is agent dependent, so it can have $A$-depth $k$ and $B$-depth $j$ in terms of the nestings of modalities. A formula is $i$-objective if its $i$-depth is $0$ and objective if both its $A$- and $B$-depths are $0$. The $k$-structure might be used with a subscript to denote the agent possessing that mental state and a superscript to represent the depth of the modal operators.

We denote the set of worlds by $\mathcal{W}$ and the set of $k$-structures for an agent $A$ by $e^k_A$. A world $w \in \mathcal{W}$ is a function from the set $P$ to $\{0,1\}$, i.e., a world stipulates which propositions are true, such that if $w[p] =1$ then $p$ is \textit{true} at the world, and \textit{false} otherwise. A $k$-structure models an agent's knowledge using the possible worlds approach: an agent $A$ say, knows the statements that are true in all the worlds they consider possible. To account for what agent $A$ knows about $B$'s knowledge, every possible world of $A$ is additionally associated with a set of worlds that $A$ knows $B$ to consider possible.
\begin{DEFINITION}
[$k$-structure \citep{bellejair}] A $k$-structure $e^k$, where $k\geq 1$ is defined inductively as: 

 \item[] \hspace{0.5cm} $\bullet$ $e^1 \subseteq \mathcal{W} \times \{\{\}\}$; and  

 \item[] \hspace{0.5cm} $\bullet$ $e^k \subseteq \mathcal{W}$ $\times E^{k-1}$, where $E^k$ is the set of all $k$-structures.

\end{DEFINITION}

Therefore, with two agents $\{A, B \}$, a $(k,j)$-model is a triple $(e^k_A, e^j_B, w)$, where $e^k_A$ is a $k$-structure for $A$, $e^j_B$ is a $j$-structure for $B$ and $w$ is a world. 
Before introducing satisfaction, we need to talk about the compatibility of worlds after an observation, and this is agent-specific\footnote{Our theory of knowledge is based on knowledge expansion where sensing ensures that the agent is more certain about the world \citep{scherllevesque2003}.}. For any two worlds $w, w'$ and observation $\rho \in OBS \cup \{\langle\rangle\}$, let us define $w\sim ^i_\rho w'$ iff $w[obs_i(\rho)] = w'[obs_i(\rho)] = 1$. That is, agent $i$ considers $w$ and $w'$ to be compatible iff $i$'s sensory data for the observation $\rho$ is the same in both. When $\rho=\langle\rangle$, then $w\sim ^i_{\rho} w'$ holds, that is, all worlds are compatible if no sensing has happened. Essentially, either $\rho$ is an action or the empty sequence where no observation has taken place.  Now, we define satisfaction following the work of Belle and Lakemeyer \citep{bellejair} but modified to account for the adaptations we have introduced above:

\begin{DEFINITION}
[Satisfaction]  For any  $z\in OBS \cup \{\langle\rangle\}$, we determine whether a formula is \textit{true} or \textit{false} after receiving the observation $z$, written as $(e^k_A, e^j_B, w, z) \models \alpha$ and defined as follows:

\begin{itemize}
        \item[] \hspace{0.5cm} $\bullet$ $e^k_A, e^j_B, w,z \vDash p$ iff $w[p] = 1$, if $p$ is an atom;
        \item[] \hspace{0.5cm} $\bullet$ $e^k_A, e^j_B, w,z \vDash \neg \alpha$ iff $e^k_A, e^j_B, w,z \nvDash  \alpha$;
        \item[] \hspace{0.5cm} $\bullet$  $e^k_A, e^j_B, w,z \vDash \alpha \vee \beta$  iff $e^k_A, e^j_B, w,z \vDash \alpha$ or  $e^k_A, e^j_B, w,z \vDash \beta$;
        \item[] \hspace{0.5cm} $\bullet$  $e^k_A, e^j_B, w,z \vDash [\rho] \alpha$ iff $z=\langle\rangle$ and  $e^k_A, e^j_B, w, \rho \vDash  \alpha$; 
        \item[] \hspace{0.5cm} $\bullet$  $e^k_A, e^j_B, w,z \vDash K_A \alpha$ iff for all $w'  \in \mathcal{W}$, $w' \sim^A_z w$ and for all $e^{k-1}_B \in E^{k-1}$, if  $(w', e^{k-1}_B) \in e^k_A$ then $e^k_A, e^{k-1}_B, w', \langle\rangle  \vDash \alpha$;  
        \item[] \hspace{0.5cm} $\bullet$   $e^k_A, e^j_B, w,z \vDash O_A \alpha$ iff for all $w' \in \mathcal{W}$, $w' \sim^A_z w$ and for all $e^{k-1}_B \in E^{k-1}$,  $(w', e^{k-1}_B) \in e^k_A$ iff $e^k_A, e^{k-1}_B, w',\langle\rangle  \vDash \alpha$;

            \item[] \hspace{0.5cm} $\bullet$  $e^k_A, e^j_B, w,z \vDash K_B \alpha$ iff for all $w'  \in \mathcal{W}$, $w' \sim^B_z w$ and for all $e^{j-1}_A \in E^{j-1}$, if  $(w', e^{j-1}_A) \in e^j_B$ then $e^k_A, e^{j-1}_B, w', \langle\rangle  \vDash \alpha$;  
        \item[] \hspace{0.5cm} $\bullet$   $e^k_A, e^j_B, w,z \vDash O_B \alpha$ iff for all $w' \in \mathcal{W}$, $w' \sim^B_z w$ and for all $e^{j-1}_A \in E^{j-1}$,  $(w', e^{j-1}_A) \in e^j_B$ iff $e^k_A, e^{j-1}_B, w',\langle\rangle  \vDash \alpha$;
       
\end{itemize}
\end{DEFINITION}

So the main difference between only-knowing and knowing is the ``iff'' rather than ``if'' which forces every pair of world and $(k-1)$-structure where $\alpha$ is satisfied to be included, and only these to be included in $e^k_A.$ Note also that on evaluating the epistemic operators, when $z = \langle\rangle$, the compatibility can be fully ignored. But when $z\in OBS$, we then look at compatible worlds and evaluate $\alpha$ in the corresponding models against the empty sequence.\footnote{For the case where $k=1$, the recursive level for agent $B$ will be of depth $0$, which corresponds to classical entailment.} We say $\alpha$ is satisfiable if there is a model of appropriate depth, and valid if it is true in every such model. Note that it is the property of the logic that if $\alpha$ is of depth $k$, and it is valid in all $k$-structures, it is also valid in all $k+n$-structures for $n\in \{0,1, \ldots\}$. So for all intents and purposes, we can stop caring about the depth of formulas as any class of semantic structures of the corresponding or higher depth suffice for reasoning. We often write $e_A^k, e_B^j, w\models \alpha$ to mean $e_A^k, e_B^j, w, \langle\rangle\models \alpha$.

\section{Sensing}
We model the agent receiving information about the world through observations. The observations received are represented as $[\rho]$, where $\rho$ is an action standing for a propositional conjunction drawn from $OBS$, interpreted, say, as reading from a sensor.  We now need to discuss how information from the sensors can be incorporated into the knowledge of the agent in a formal way. Recall that since we start with agents only-knowing formulas, we cannot simply conjoin new knowledge. That is why we leverage an insight from the dynamic multi-agent only knowing framework: what the agent knows after sensing is contingent on what was known previously in addition to observing some truth about the real world \citep{bellejair}. In other words, the following is a theorem in the logic: 

\begin{THEOREM} [Sensing \citep{bellejair}, Th.19] \label{theorem_sensing}


Given objective formulas $\Sigma, \Sigma',$ and $\Gamma$; a formula $\alpha$ that is either propositional or at most mentions $K_i$ operators; and an observation $\rho$, then:

$  
\Gamma \land O_A(\Sigma \wedge O_B \Sigma' ) \vDash [\rho] K_A \alpha 
$ iff $\Gamma \land O_A(\Sigma \wedge O_B \Sigma' ) \vDash obs_A(\rho) \wedge K_A(obs_A (\rho) \rightarrow [\rho]\alpha)$. 
\end{THEOREM}

The proof is based on the fact that when the sensing action happens, we only look at worlds that agree with the sensed observation (in accordance with the semantics for the dynamic operator). Therefore not only must the sensed observation hold in the real world, but also knowing this observation must mean that $\alpha$ is known (assuming it was not known already). Note two points: first,  the observations are assumed to not conflict with what was already known or with $\Gamma$, which represents the real world; that is, as discussed before, we are operating under the setting of knowledge expansion and not belief revision. The agent may be ignorant about something and the sensing adds more knowledge to the agent, but it is not possible for the agent to know something which is then contradicted by an observation. This is the standard setting in the epistemic situation calculus \citep{scherllevesque2003}. Second, the sensing theorem works very much like the successor state axiom for knowledge in the epistemic situation calculus \citep{scherllevesque2003}; however, in the latter, it is a stipulation of the background theory, whereas  here it is a theorem of the logic. The sensing theorem establishes that  $[\rho]K_A\alpha \equiv o \land K_A(o \rightarrow[\rho]\alpha)$, where $o = obs_A(\rho)$, and it is the RHS that we will make use of in our learning theorem. Note, however, that in the RHS, the dynamic modality is now being applied to $\alpha.$ At this point, the sensing theorem applies recursively and stops when it is in the context of a propositional formula. That is, $[\rho]\alpha = \alpha$ if $\alpha$ is propositional; and $[\rho]\alpha = o' \land K_i(o'\rightarrow [\rho]\beta)$  if $\alpha  = K_i \beta$, where $o' = obs_i (\rho) $.

This is the essence of the Regression Theorem \citep{bellejair}, where the application of an observational action in the context of a propositional formula yields the formula itself because sensing does not affect truth in the real world. Only when it encounters an epistemic operator, it uses the RHS of the sensing theorem. In what follows, for improving readability, we abuse notation and sometimes use $\rho$ outside the dynamic operator to mean the corresponding observation w.r.t.\ the root agent. That is, we write a formula such as $\rho \land \alpha$ to mean $obs_A(\rho) \land \alpha$. Likewise, we write $K_A(\rho \supset [\rho]\alpha)$ to  mean $K_A(obs_A(\rho) \supset [\rho]\alpha)$, that is, inside the dynamic operator $\rho$ is left as is but everywhere else it is being replaced by the observational formula that it corresponds to.

\section{Reasoning}
The language in general allows for arbitrary nesting of epistemic operators. And as already mentioned, at least for formulas not mentioning the dynamic operator and only knowing, $k$-structures can be shown to be semantically equivalent to the {Kripke} structures with respect to the entailment of a formula. What we are interested in now is finding a connection between validity and what we require in the learning algorithm. To do this we need to resolve two issues: first, how can observations be incorporated into the background knowledge, and second, how can the entailment of the query with respect to the background knowledge as well as the observations be evaluated? 

It is important to appreciate that the second challenge deserves great attention because we are dealing with noisy observations. Roughly speaking, the way implicit learning works \citep{juba2013} given a set of noisy observations is that the conjunction of the background knowledge together with the observation is used to check if the query formula logically follows. Suppose this happens for a high proportion of the observations. In that case, the query is accepted by the decision procedure which can be seen as implicitly including whatever formula might be captured by the high proportion of observations. So checking logical validity will be an important computational component of the overall algorithm. We will only obtain a polynomial time learning algorithm if checking validity is in polynomial time. It is widely known that reasoning in the \textit{weak-S5} is PSPACE-complete \citep{HansHalpern2015}. So what hope do we have? Not surprisingly, with the only-knowing operator, the reasoning is much harder, at the second level of the polynomial hierarchy even for a single agent \citep{rosati2000}. However, as it turns out, there is a very popular and interesting result: if one is interested only in the validity of the formulas $O(\Sigma) \rightarrow K \alpha$, this can be reduced to propositional reasoning. (However, this, in fact, is \textit{co-NP hard} \citep{Levesque2001logic}, but that's another matter because much is known about bounded proofs in propositional logic \citep{juba2012, juba2013}. Likewise, when we consider multiagent knowledge bases of the form 
$O_A(\phi \wedge O_B(\psi \wedge...) ...)$, where $\phi$ and $\psi$ are objective formulas, and we are interested in the entailment of $K_A \alpha$, where $\alpha$ does not mention dynamic ``$[\cdot ]$'' nor ``$O_i$'' operators, we can reduce it to propositional reasoning \citep{bellejair}. The reduction to propositional reasoning is achieved using the \textit{Representation Theorem} denoted by the operator $|| \cdot ||$, first introduced by Levesque and Lakemeyer \citep{Levesque2001logic}. It works by going through a formula and replacing knowledge of an objective sentence by either \textit{true/false} according to whether the sentence is entailed by the given knowledge base $\Sigma$. This idea is then generalized to non-objective knowledge by working recursively on formulas from the inside out.

\begin{DEFINITION}
    [Representation Theorem \citep{bellejair}, simplified from Def.25]  Let $\phi$ and $\psi$ denote the set of sentences only known by agent $A$ and $B$ respectively. Then for any epistemic formula $\alpha$, $|| \alpha||_{\phi, \psi} $ is defined as follows:

    \begin{itemize}
     \item[]   \hspace{0.5cm}  1. $|| \alpha||_{\phi, \psi} = \alpha$, if $\alpha$ is objective;
        \item[] \hspace{0.5cm}  2. $|| \neg \alpha||_{\phi, \psi}  = \neg || \alpha||_{\phi, \psi} $;
        \item[] \hspace{0.5cm}  3. $|| \alpha \vee \beta ||_{\phi, \psi}  = ||\alpha||_{\phi, \psi} \vee ||\beta|| _{\phi, \psi} $;
        \item[] \hspace{0.5cm}  4. $|| K_A \alpha || _{\phi, \psi}  = \phi \rightarrow ||\alpha||_{\phi,\psi}$; 
        \item[] \hspace{0.5cm}  5. $|| K_B \alpha || _{\phi, \psi}  = \psi\rightarrow ||\alpha||_{\phi,\psi}$. 
    \end{itemize}

\end{DEFINITION}

Putting it together, suppose that $\phi, \psi$ are objective formulas and $\alpha$ is an epistemic formula that does not mention $\{ O_i, [\cdot]\}$ operators. Then $O_A(\phi \wedge O_B \psi) \vDash K_A \alpha$ if $\vDash ||\alpha||_{\phi, \psi}$, where $||\alpha||_{\phi, \psi}$ is a propositional formula. The reduction works by slicing up the knowledge base and query at the modal operator and transforming these formulas into objective formulas. 

Originally, the representation theorem used a second operator $RES[\cdot,\cdot]$ \citep{LEVESQUE1990alliknow} in Cases 4 and 5, but for propositional languages, this operator simplifies to checking entailment w.r.t.\ the indicated agent's knowledge base.

\begin{EXAMPLE}
  Suppose we have a query  $\alpha$ of the form $\alpha = K_A K_B p$ and $\phi, \psi$ are sets of sentences believed by the agent $A$ and $B$ respectively, i.e. $O_A(\phi \wedge O_B(\psi))$. Since $\alpha$ contains modal operators, we apply Cases 4 and 5 recursively: first, since the query has $K_A$ at the outermost position, we refer to the knowledge base believed by $A$ in Case 4, obtaining $||K_A K_B p || _{\phi, \psi}=\phi\rightarrow ||K_B p|| _{\phi, \psi}$. Then, we recursively apply Case 5 to obtain $||K_B p|| _{\phi, \psi}=\psi\rightarrow ||p||_{\phi, \psi}$.

  And then  because $p$ is an atom, which is objective, $||p||_{\phi, \psi}= p$,  so $||K_A K_B p || _{\phi, \psi}=\phi\rightarrow ||K_B p|| _{\phi, \psi}=\phi\rightarrow (\psi\rightarrow ||p||_{\phi, \psi})$

which finally results in checking the validity of $\phi \rightarrow (\psi \rightarrow p)$, or equivalently $(\phi \wedge \psi) \rightarrow p$.
  
\end{EXAMPLE}

\begin{COROLLARY}
    Suppose $\phi, \psi$ are sets of sentences believed by the agent $A$ and $B$ respectively  and $\alpha$ is an epistemic formula from $\mathcal{OL}_n$ that does not mention the $\{O_i, [\cdot] \}$ operators. Suppose $\rho$ is a propositional formula. Then, $O_A(\phi \wedge O_B \psi) \wedge \rho \vDash \alpha$ iff $\rho \rightarrow || \alpha || _{\phi, \psi}$, where $|| \cdot ||$ is as above.
\end{COROLLARY}

The following example  is adapted and modified from \citep{bellejair}: 

\begin{EXAMPLE}[Card game] Suppose two agents $A$ and $B$ are playing a card game with cards numbered from $1$ to $4$. The cards have been shuffled and two face-down cards are dealt, one to each agent. A player picks a card, reads the number on it and has to decide whether to challenge the other player or not. Once the other player responds by showing their card, the player with the highest number on the card wins.

We use the notation $N_A=\#1$ to represent that agent $A$ has drawn card number $1$ from the deck, and analogously for agent $B$ we use $N_B$. The initial conditions $\Sigma$ are represented by:
 \begin{itemize}
     \item[] \hspace{0.5cm} 1. $((N_{A} = \#1) \vee (N_{A} = \#2) \vee (N_{A} = \#3) \vee (N_{A} = \#4)) \wedge ((N_{B }= \#1) \vee (N_{B }= \#2) \vee (N_{B}= \#3) \vee (N_{B} = \#4))$: each agent draws only one card from the deck;
    \item[] \hspace{0.5cm} 2. $(N_{A} = \#4) \equiv  (N_{B} = \#1) \vee (N_{B} = \#2) \vee (N_{B} = \#3)$: the card agent $A$ draws must be distinct from agent $B$'s card (and likewise for other combinations);
    \item[] \hspace{0.5cm} 3. $W_A \equiv (N_A = \#2 \wedge N_B = \#1) \vee (N_A = \#3 \wedge N_B = \#1) \vee (N_A = \#4 \wedge N_B = \#1) \vee (N_A = \#3 \wedge N_B = \#2) \vee (N_A = \#4 \wedge N_B = \#2) \vee (N_A = \#4 \wedge N_B = \#3)$;
    Agent $A$ wins the round if he draws a card with a higher number than agent $B$, and analogously when agent $B$ is winning;
     \item[] \hspace{0.5cm} 4. $(\neg W_A \wedge \neg W_B)$: initially, no agent has won the game; and 
     
     \item[] \hspace{0.5cm} 5. finally, we also need to introduce the observational actions. For every card picking action $N_i = \#n$, we assume an action $\rho_{in}$.  Clearly, such an action should tell agent $i$ that his card is $n$, but should not reveal anything else to the other agent. (Interestingly, by the above formulas, the agent $i$ should be able to infer that the other agent has any card but $n$; we will come to this later.) So, we define $obs_i(\rho_ {in}) = (N_i = \#n)$ but $obs_j(\rho_{in}) = true$ for $j\neq i.$ By extension, define the action $\rho_{in,jk}$ to mean $i$ has read the value $n$ and $j$ the value $k$, and thus,  $obs_i(\rho_{in,jk}) = n$ and $obs_j(\rho_{in,jk}) = k.$
     
 \end{itemize}

Both agents $A$ and $B$ have the same initial knowledge about the game, encoded by a formula $\phi$ which also includes the sensing rule described above.
Let's say that agent $A$ draws card $\#4$ and agent $B$ draws card $\#3$. These observations are represented by $(N_{A} = \#4) \wedge (N_{B}=\#3)$. Since we are making the case for the root agent to be $A$ we have the initial theory as $\theta = O_A(\phi \wedge O_B(\phi)) \wedge [(N_{A} = \#4) \wedge (N_{B}=\#3)]$. We can then reason about beliefs and non-beliefs. The following properties follow:
 
\begin{enumerate}
\item[] \hspace{0.5cm} 1. $\theta \vDash \neg K_A (N_A = \#1) \wedge \neg K_A (N_A = \#2) \wedge \neg K_A (N_A = \#3) \wedge \neg K_A (N_A = \#4)$;
     
Initially, agent $A$ does not know the card he has.
\item[] \hspace{0.5cm} 2. $\theta \vDash [\rho_{A4}] K_A (N_{A}=\#4)$; After A picks up  his card, and then sensing the card, agent $A$ knows that the number is 4.

\item[] \hspace{0.5cm} 3. $\theta \vDash [\rho_{A4}] K_A \neg (N_{B}=\#4)$; The logical reasoning allows $A$ to infer that $N_B$ is one of $\{\#1, \#2, \#3\}.$

\item[] \hspace{0.5cm} 4. $\theta \vDash [\rho_{A4}] K_A (\neg  K_B ((N_A = \#1))$; Agent $A$ knows that agent $B$ does not know the number on $A$'s card, by means of his knowledge of the sensing actions (likewise for the $N_A = \# 2, N_A = \#3, N_A = \#4$).

\item[] \hspace{0.5cm} 5. $\theta \vDash [\rho_{A4}] K_A K_B(K_A (N_A = \#1) \vee K_A(N_A \neq \#1))$; Agent $A$ knows that agent $B$ does not know the number of his card, but nonetheless knows that $A$ knows whether he has the card.

\item[] \hspace{0.5cm} 6. $\theta \vDash [\rho_{A4}] K_A W_A \land K_A \neg K_B W_A$; By logical reasoning, $A$ infers that he has won but knows that $B$  does not know this.

\item[] \hspace{0.5cm} 7. $\theta \vDash [\rho_{A4,B3}] K_A (K_B\neg (N_{B}=\#4))$;

After both agents see their  cards, $A$ knows that $B$ knows his card, which cannot be $N_{B} = \#4$ since $A$ has the card with the number 4. At this point in the game, if $B$ had obtained the card with value 1, only then would he know he has lost. But since he has the value 3, he does not know that he has lost, because, after all, $A$ could have the card with value 2. By extension, $A$ also does not know that $B$ does not know. This is because only if $B$'s card had the value 1, $B$ knows that he has lost. But in all other circumstances, $B$ could still imagine that he has the higher card.





 \end{enumerate}

\end{EXAMPLE}


\section{PAC-Semantics}
Various approaches have been proposed in the attempt to gain efficient and robust learning, including inductive logic programming \citep{Muggleton1994} and concept learning \citep{valiant84}.
Concept learning, also known as \textit{Probably Approximately Correct} (PAC) learning is a machine learning framework where the classifier receives a finite set of samples from distribution and must return a generalisation function within a class of possible functions. This approach aims to produce with high probability (P) a function that has a low generalisation error (AC). In the context of logic, Valiant \citep{Valiant2000} proposed the \textit{PAC Semantics}, a weaker semantics (compared to the classical entailment) for answering queries about the background knowledge, that integrates noisy observations against a logical background knowledge base. The knowledge is represented on one hand by a collection of axioms about the world, and on the other, by a set of examples drawn from an (unknown) distribution. In this way, the algorithm uses both forms of knowledge to answer a given query, which one may not be able to answer using only the background knowledge or the standalone examples. The output generated by this approach does not, however, capture the sense of validity in the standard (Tarskian) sense; rather, validity is defined as follows:

\begin{DEFINITION}
[$(1-\epsilon)$-validity] Suppose we have a distribution $\mathcal{D}$ supported on $E_A ^k\times E_B ^j \times \mathcal{W}$ and $\alpha$ an epistemic formula $\alpha \in \mathcal{OL}_n$, that does not mention $\{ O_i, [\cdot]\}$ operators. Then we say that $\alpha$ is $(1-\epsilon)$-valid iff $Pr_{(e^k_A, e^j_B, w) \in E_A ^k\times E_B^j \times \mathcal{W}} [(e^k_A, e^j_B, w) \vDash \alpha] \geq 1- \epsilon$, where $(e^k_A, e^j_B, w)$ is a model from $E_A ^k\times E_B ^j \times \mathcal{W}$. If $\epsilon  = 0$, we say that $\alpha$ is perfectly valid.
\end{DEFINITION}

It is worth pointing out that the overall thrust of the framework is quite different from popular approaches such as inductive logic programming \citep{Muggleton1994} and statistical relational learning \citep{getoor2007introduction,de2011statistical} etc. In inductive logic programming, an explicit hypothesis is learned that captures only the examples currently provided, and rarely is there an explicit account of how the hypothesis might reflect some unknown distribution from which the examples are drawn. In statistical relational learning, a hypothesis captures the distribution of the examples provided in terms of a logical formula or a probabilistic logical formula.

Once again, no accounts are provided about how this formula might capture the unknown distribution from which the examples are drawn. The key trick is to develop a decision procedure that checks the entailment of the query against the knowledge base by conjoining the observations to the background theory and checking the queries. Then, if for a higher proportion of observations, the query is indeed contained, we conclude that the explicit knowledge base together with the implicit knowledge base entails the query and all of this is robustly formalised using the \textit{PAC-semantics} \citep{Valiant2000}. The proportion of times the query formula evaluates to \textit{true} can be used as a reliable indicator of the formula's degree of validity, as guaranteed by Hoeffding's inequality \citep{hoeffding1963probability}. 

\begin{THEOREM}[Hoeffding's inequality]
Let $X_1,\ldots,X_m$ be independent random variables taking
values in $[0,1]$. Then for any $\epsilon>0$,
$$
\Pr\left[\frac{1}{m}\sum_{i=1}^mX_i\geq \mathbb{E}\left[\frac{1}{m}\sum_{i=1}^mX_i\right]+\epsilon\right]\leq e^{-2m\epsilon^2}.
$$
\end{THEOREM}

The agent will have some knowledge base encoded in the system, and will also be able to sense the world around them and receive readings describing the current state of the world. These readings are generally correct for that environment but are neither fully accessible nor are they exact. As a consequence, the observations can be noisy or inconsistent with each other, but they are always consistent with the knowledge base. So in the spirit of knowledge expansion, the agent may be ignorant about many things and the agent is informed by the observations it makes. These observations only focus on a few properties of the world, as would be expected from most physical sensors in robotics and other applications. Formally we introduce a \textit{masking process} that randomly reveals only a few properties of the world \citep{Loizos2010}. These readings are conjunctions of propositional atoms and are drawn independently at random from some probability distribution $\pmb{M}$ over $\mathcal{L}_n$ which is unknown to the agent. For example, a smartwatch might only be getting readings about the heart rate of the person wearing the watch but not the blood oxygen levels \citep{rader2021}. In the multi-agent setting, sensors may only reveal what cards the agent itself holds but may not provide information about the cards held by the other agents as is intuitive.  An agent attempting to sense the same environment twice could end up with two different observations, and so the masking process captures this stochastic nature of sensing.

\begin{DEFINITION} [Masking Process] 

The masking process $\pmb{M}$ is a random mapping from a model $(e_A^k,e_B^j,w)$ to a propositional conjunction $\rho^{(c)} \in \mathcal{L}_n$ such that $(e_A^k,e_B^j,w) \vDash \rho^{(c)} $. Then $\pmb{M} (\mathcal{D})$ is a distribution over the set $OBS$, induced by applying the masking process to a world drawn from $\mathcal{D}$. 
    
\end{DEFINITION}

The masking process induces a probability distribution $\pmb{M}(\mathcal{D})$ over observations $\rho \in OBS$ modelling the readings of an agent's sensors. We assume from now on that this modelling is performed on the root agent's sensors. The masking process can be understood in two different ways: either the readings are absent due to a stochastic device failure, or the agent is unable to concurrently detect every aspect of the state of the world. In essence, the model incorporates unforeseen circumstances like a probabilistic failure or an agent's sensors' inability to provide readings \citep{Ferreira2005}. So the formalization of the reasoning problem should capture this limitation somehow. The reasoning problem of interest becomes deciding whether a query formula $\alpha$ is $(1-\epsilon)$-valid. Knowledge about the distribution $\mathcal{D}$ comes from the set of examples $\rho^{(c)} \in \mathcal{L}_n$. Additional knowledge comes from a collection of axioms, the knowledge base $\Sigma$. We do not have complete knowledge of the models drawn from $\mathcal{D}$, instead, we only have the observations $\rho$ sampled from $\pmb{M}(\mathcal{D})$ and the knowledge base $\Sigma$. 

We assume here two agents $A$ and $B$, but it can be generalized to multiple agents, from which agent $A$ is the root agent. The background knowledge is represented by $\Sigma, \Sigma' \in \mathcal{OL}_n$, where $\Sigma$ corresponds to agent $A$ and $\Sigma'$ corresponds to agent $B$. The input query $\alpha$ is of the form $M^l \alpha'$, where $M$ denotes a sequence of bounded modalities, i.e. $K_A K_B K_A \alpha ' = M^3 \alpha'$, the query is of maximal depth of $k, j$ which are the $i$-depths of the $k$-structures for agents $A$ and $B$. And finally, we draw $m$ partial observations which are of propositional format $\rho^{(1)},\rho^{(2)}, ..., \rho^{(m)}$. In implicit learning, the query $\alpha$ is answered from observations directly, without creating an explicit model. This is done by means of entailment: we repeatedly ask whether $O_A(\Sigma) \vDash [\rho^{(c)}] K_A  \alpha$ for examples $\rho^{(c)}\in \pmb{M}(\mathcal{D})$ where $c \in \{1,..., m\}$. So this entailment checking with respect to each observation $\rho^{(c)}$ becomes our best approximation to $(1-\epsilon)$-validity. If at least $(1-\epsilon)$ fraction of the examples entail the query $\alpha$, the algorithm returns \textit{Accept}. The estimation is more accurate the more examples we use. The concepts of accuracy and confidence are captured by the hyper-parameters $\gamma, \delta \in (0,1)$, where $\gamma$ bounds the accuracy of the estimate and $\delta$ bounds the probability the estimate is invalid.

\begin{DEFINITION}[Witnessed formula] 

The implicit knowledge for agent $A$'s  mental state is the set $I$ of $\alpha$'s such that with probability $1-\epsilon$ over $\rho$ drawn from $\pmb{M}(\mathcal{D})$ (i.e., for a model drawn from $\mathcal{D}$ and passed through $\pmb{M}$),  $O_A(\Sigma) \vDash [\rho] K_A(\alpha)$. We say that $\alpha$ is \emph{witnessed true} on $\rho$ in this event.
\end{DEFINITION}
 
Since we are only concerned with entailment and not proof theory here, there is essentially no distinction between the implicit knowledge itself and the provable consequences/entailments of the implicit knowledge. We can just talk about whether or not $\alpha$ is implicitly known to $A$. These formulas are witnessed true with probability $1- \epsilon$; in particular, a proportion of them may still evaluate to \textit{false} with probability up to $\epsilon$. We will now move on and motivate the learning algorithm. As observed before, the key step in the decision procedure is given the partial observations as discussed above and the background theory which is of the form $O_A(\phi \wedge O_B(\psi ...))$ or simply $O_A(\Sigma)$, we check the entailment of the query given the observations against the background theory. As we discussed in Theorem \ref{theorem_sensing}, this can be reduced to a statement of the following form: $O_A(\Sigma) \vDash \rho \wedge K_A(\rho \rightarrow \alpha)$.

\begin{algorithm}[h]
\KwIn{ $\Sigma$ set of sentences from root agent $A$; input query $\alpha$; partial observations: $\rho^{(1)}, \rho^{(2)}, ..., \rho^{(m)}$; hyper-parameters: $\epsilon, \gamma, \delta \in (0,1)$.
}

\KwOut{\textbf{Accept} if there exist formulas $I$ witnessed \textit{true} with probability at least $(1-\epsilon + \gamma)$ on $\pmb{M}(\mathcal{D})$ such that $O_A(\Sigma) \vDash K_A(I\rightarrow\alpha)$ OR \textbf{Reject} if $\Sigma  \vDash \alpha$ is not  $(1-\epsilon - \gamma)$-valid under the distribution $\mathcal{D}$.}

\BlankLine
\Begin{
$b \leftarrow \lfloor \epsilon \times m \rfloor , FAILED \leftarrow 0$.\\

\ForEach{$c$ in $m$} {
\If{$\rho^{(c)} \wedge O_A(\Sigma) \nvDash  K_A(\rho^{(c)} \rightarrow \alpha)$
}
{
Increment $FAILED$.
\\
\If{$FAILED > b$}{\bf return Reject}}
}
\bf return  Accept}
\caption{ DecidePAC Implicit learning reduction}\label{algo:il}
\end{algorithm}

\begin{THEOREM} 
[Multi-agent implicit epistemic learning]\label{thm:ilmas}
Let $\Sigma$ be the epistemic knowledge base of a root agent $A$, and suppose $\Sigma$ is perfectly valid for $\mathcal{D}$. 
We sample $m = \frac{1}{2\gamma^2}\ln\frac{1}{\delta}$ observations $\{\rho^{(1)}, \rho^{(2)}, ..., \rho^{(m)}\}$ from $\pmb{M}(\mathcal{D})$, which represent the partial information sensed by the root agent $A$. Then, with probability $1-\delta$:
\begin{itemize}
    \item[] \hspace{0.5cm}  1. If $\Sigma\rightarrow \alpha$ is not $(1- \epsilon - \gamma)$-valid with respect to the distribution $\mathcal{D}$, Algorithm \ref{algo:il} returns \textbf{Reject}; 
    \item[] \hspace{0.5cm} 2. If there exists some implicit knowledge $I$ such that $\beta\in I$ is witnessed true with probability at least $(1 - \epsilon + \gamma)$ and $O_A(\Sigma) \vDash K_A(\beta\rightarrow\alpha)$ then Algorithm \ref{algo:il} returns \textbf{Accept}.
 
\end{itemize}
\end{THEOREM}

\begin{proof} \textbf{Case 1.} We assume that $\alpha$ is not $(1-\epsilon - \gamma)$-valid and show that the algorithm rejects with probability $1-\delta$. By definition, if $\Sigma\rightarrow \alpha$ is not $(1-\epsilon - \gamma)$-valid, then since $\Sigma$ is consistent with worlds sampled from $\mathcal{D}$, $\alpha$ is false with probability at least $\epsilon + \gamma$; in turn, since an observation $\rho^{(c)}$ produced by $\pmb{M}$ is consistent with this world, it also must not entail $\alpha$ with probability at least $\epsilon + \gamma$, and thus, $\Sigma \nvDash \rho^{(c)}\rightarrow\alpha$. In turn, then, $\rho^{(c)} \wedge O_A(\Sigma) \nvDash  K_A(\rho^{(c)} \rightarrow \alpha)$. It follows from Hoeffding’s inequality now that for $m = \frac{1}{2\gamma^2}\ln\frac{1}{\delta}$, the number of failed checks is at least $\epsilon m$ with probability at least $1- \delta$, so DecidePAC returns \textbf{Reject}.

\noindent
\textbf{Case 2.} Assume that there exists such implicit knowledge $I$ such that $\beta\in I$ is witnessed \textit{true} with probability at least $1-\epsilon +\gamma$. By definition, when $\beta$ is witnessed, $\rho^{(c)}\land O_A(\Sigma)\models [\rho^{(c)}]K_A(\beta)$. We also have, by assumption, that $O_A(\Sigma)\models K_A(\beta\rightarrow\alpha)$, and hence $\rho^{(c)}\land O_A(\Sigma)\models [\rho^{(c)}]K_A(\alpha)$. By Theorem~\ref{theorem_sensing}, this may be reformulated as $\rho^{(c)}\land O_A(\Sigma)\models \rho^{(c)}\land K_A(\rho^{(c)}\rightarrow \alpha)$. Thus, Hoeffding's inequality gives that for $m = \frac{1}{2\gamma^2}\ln\frac{1}{\delta}$, $\beta$ is witnessed in at least $(1-\epsilon) m$ observations with probability $1-\delta$, and in turn $\rho^{(c)}\land O_A(\Sigma) \nvDash K_A(\rho^{(c)}\rightarrow \alpha)$ in fewer than $\epsilon m$ observations. The algorithm will therefore return \textbf{Accept} with probability $1-\delta$.
\end{proof}
Let us demonstrate the functionality of the PAC framework by going through the example previously outlined and considering a real-life scenario where some statement that often (but not always) holds may be inferred from the information available to the agent.

\begin{EXAMPLE}
[Card game with partial observations] Similarly to the previous example, we have the same players and card rules. Let's start with the first property outlined earlier where after sensing the card, agent A can reason about what card it holds. The initial knowledge base is the same as previously: $\theta = O_A(\phi \wedge O_B(\phi))$, and suppose that in four different games, the observations are received in series as follows: $\{ \rho_A^{(1)}: N_A = \#4, \rho_A^{(2)}: N_A = \#3, \rho_A^{(3)}: N_A = \#4, \rho_A^{(4)}: N_A = \#4  \}$. We will take as an example the second property outlined above, where the root agent is asked to reason whether he will know that his card number is $\#4$, that is $\theta \vDash [\rho] K_A (N_{A}=\#4)$. For the first observation, $\rho^{(1)}$, the entailment is straightforward since the observation is consistent with the query about the agent's knowledge. The entailment problem for the second observation $\rho^{(2)}$ becomes $\theta \vDash [ N_A = \#3] K_A (N_{A}=\#4)$. Once $N_A$ is observed, then the question is whether there exist some worlds $w' \in W$ such that $w'[N_A = \#3] = w[N_A = \#4]$. Such worlds do not exist, since agent A could not have picked both cards 3 and 4 at the same time, so the query is not entailed, and $FAILED$ increments. Similarly, the next two observations entail the query. After iterating through every observation the DecidePAC algorithm determines the degree of validity of the query, $(1- \epsilon)$-validity. For this case, if $\epsilon$ was set to a value of $0.25$, then the algorithm returns Accept with $0.75$-validity for this query.

\end{EXAMPLE}

\section{Tractability and Future Work}
One of the main motivations of implicit learning, following the works of Khardon and Roth on learning-to-reason framework \citep{KhardonRothJournal1997}, and Juba on implicit learnability \citep{juba2013}, was to enable tractable learning for reasoning by bypassing the intractable step of producing an explicit representation of the knowledge. Indeed, if the observations are sufficiently nice, reasoning with the PAC-semantics may even be more efficient than classical reasoning \citep{juba2015restricted}. The problem is only more acute in multi-agent settings: in prior work on multi-agent reasoning \citep{Lakemeyer2012EfficientRI}, Lakemeyer proposed a polynomial time algorithm if the knowledge base was encoded as a proper epistemic knowledge base which proved to be computationally costly. Indeed, most prior work on efficient multi-agent reasoning requires such an expensive compilation step in order to introduce a conjunctive observation to an agent's knowledge base. By using the Representation Theorem \citep{bellejair}, we can reduce the entailment checks in Algorithm~\ref{algo:il} to propositional queries. In turn, if these propositional queries are members of an adequately restricted fragment such as Horn clauses or more generally, those provable using bounded-space treelike resolution, then polynomial-time algorithms exist for deciding such queries~\citep{kullmann99}. Similarly, Liu et al. \citep{Liu2004ALO} proposed a sound and (eventually) complete method for reasoning in first-order logic that is polynomial in the size of the knowledge base. The tractability is guaranteed using a notion of mental effort characterised by the parameter \textit{k}, for example, if $k=0$, then a sentence $\alpha$ is believed only if it appears explicitly in the given knowledge base. The downside of this approach is that the complexity of converting the entire knowledge base to a required \textit{CNF} formula increases exponentially as the parameter $k$ increases. What remains to be done, that we leave for future work, is to establish a guarantee that the propositional reasoning needed to decide a query remains in such a tractable fragment. Alternatively, it might be sensible to take one of the limited belief reasoning logics \citep{liu2005tractable} and integrate it into our learning framework. All of these are interesting directions for future work.

\section{Related Work}
Several results have been obtained by leveraging PAC-semantics' advantages in order to obtain implicit learning. After first being formalised by Juba \citep{juba2013} for the propositional logic fragments and obtaining more efficient reasoning there \citep{juba2015restricted}, Belle and Juba \citep{bellejuba2019} demonstrated an integration of the PAC-semantics with first-order logic fragments. Later on, work by Mocanu et al. \citep{mocanu} showed that polynomial reasoning can be achieved efficiently with partial assignments for standard fragments of arithmetic theories. They proposed a reduction from the learning-to-reason problem for a logic to any sound and complete solver for that fragment of logic. 
On the empirical side, Rader et al. \citep{rader2021} proposed an empirical study for learning optimal linear programming objective constraints which significantly outperforms the explicit approach for various benchmark problems. Although polynomial guarantees are obtained for various language domain fragments, they do not offer the expressiveness of a multi-agent language, which is what our work considered.

On the epistemic logic axis, Lakemeyer and Lesperance have proposed a form of epistemic reasoning in which the knowledge base is encoded as a set of modal literals (PEKB) \citep{Lakemeyer2012EfficientRI}. Although this approach showed some computational speedup when it comes to reasoning, it did carry along the disadvantage of converting both the knowledge base and the query to certain formats before entailment is checked. This conversion becomes computationally costly as the modal depth increases, and although this form of knowledge might be useful for certain applications, it does not handle an important epistemic notion: \textit{knowing-whether} \citep{FanWD13}. That means that the language does not cover any form of incomplete knowledge or disjunctions (horn clauses) and so very limited forms of inference were possible. Although an extension to knowing-whether was later proposed in work by Miller et al. \citep{miller2016}, it still lacks arbitrary disjunctions, which our framework can handle fully. In a non-modal setting, Lakemeyer and Levesque \citep{LakemeyerLevesque2004} proposed a tractable reasoning framework for disjunctive information. In another work, Fabiano \citep{Fabiano2019MAEP} proposed an action-based language for multi-agent epistemic planning and implemented an epistemic planner based on it.  Work by Muise et al. \citep{Muise2021MAEP} addresses the task of synthesizing plans that necessitate reasoning about the beliefs of other agents. Learning is not addressed in any of these. Likewise, in research work by Lakemeyer et al. \citep{Lakemeyer2020} a logic of limited belief with a possible-worlds semantics is proposed. Each epistemic state is represented as sets of three-valued possible worlds, which allows some tractability with epistemic reasoning, but it is limited to the single agent case and also does not address learning. In the context of dynamic epistemic logic \citep{sep-dynamic-epistemic}, there is some recent work on ``qualitative learning'' \citep{Bolander2015} which considers learning in the limit for propositional action models. This is very different in thrust from ours, where we are interested in answering queries from noisy observations with robustness guarantees. Although there are many research works focusing on learning in multi-agent settings for coordination and competition \citep{stefano2017}, we are not aware of any work that addressed arbitrary nesting of belief operators in the style of $K45_n$ in a learning setting.

\section{Conclusion}
In this work, we demonstrated new PAC learning results with multi-agent epistemic logic. We considered the PAC semantics framework which integrates real-time observations from the current world along with the background knowledge in order to decide the entailment of epistemic states of the agents. We leveraged some recent results on multi-agent only knowing, namely the Representation Theorem, in order to reduce modal reasoning to propositional reasoning. We have formalised the learning process and discussed the sample complexity and correctness of an algorithm for learning. The algorithm is in principle applicable to any multi-agent logic, as long as a sound and complete procedure is used to evaluate epistemic queries against an epistemic knowledge base.
If one did not take into consideration the time complexity and allow for the full $K45_n$ reasoning instead, then we could swap the entailment checking in the algorithm for general validity in that particular logic.  However, this inevitably implies that it would become intractable. 
Considering that reasoning in the full $K45_n$ is PSPACE-complete, this gave us the incentive to focus on the only knowing angle instead, which allows us to at least reduce entailment to propositional reasoning. As discussed in the related work section, there are many promising ideas from  other approaches that might provide tractability either in the only-knowing setting (via limited reasoning \citep{liu2003tractability, liu2005tractable}) or a more general multi-agent knowledge setting, such as  \citep{Lakemeyer2012EfficientRI}.

\section*{Acknowledgements}

\bibliographystyle{acmtrans}
\bibliography{new_tlp2egui}

\begin{thebibliography}{}

\bibitem[\protect\citeauthoryear{Albrecht and Stone}{Albrecht and
  Stone}{2017}]{stefano2017}
{\sc Albrecht, S.} {\sc and} {\sc Stone, P.} 2017.
\newblock Autonomous agents modelling other agents: A comprehensive survey and
  open problems.
\newblock {\em Artificial Intelligence\/}~{\em 258}.

\bibitem[\protect\citeauthoryear{Bacchus, Halpern, and Levesque}{Bacchus
  et~al\mbox{.}}{1999}]{Bacchus1999171}
{\sc Bacchus, F.}, {\sc Halpern, J.~Y.}, {\sc and} {\sc Levesque, H.~J.} 1999.
\newblock Reasoning about noisy sensors and effectors in the situation
  calculus.
\newblock {\em Artificial Intelligence\/}~{\em 111,\/}~1--2, 171 -- 208.

\bibitem[\protect\citeauthoryear{Baltag and Renne}{Baltag and
  Renne}{2016}]{sep-dynamic-epistemic}
{\sc Baltag, A.} {\sc and} {\sc Renne, B.} 2016.
\newblock {Dynamic Epistemic Logic}.
\newblock In {\em The {Stanford} Encyclopedia of Philosophy\/}, {W}inter 2016
  ed., {E.~N. Zalta}, Ed. Metaphysics Research Lab, Stanford University.

\bibitem[\protect\citeauthoryear{Belardinelli and Lomuscio}{Belardinelli and
  Lomuscio}{2007}]{belardinelli2007}
{\sc Belardinelli, F.} {\sc and} {\sc Lomuscio, A.} 2007.
\newblock A quantified epistemic logic for reasoning about multiagent systems.
\newblock In {\em AAMAS}. 1--3.

\bibitem[\protect\citeauthoryear{Belle}{Belle}{2010}]{DBLP:conf/dagstuhl/Belle10}
{\sc Belle, V.} 2010.
\newblock Multi-agent only-knowing revisited.
\newblock In {\em AlgoSyn}. 16.

\bibitem[\protect\citeauthoryear{Belle and Juba}{Belle and
  Juba}{2019}]{bellejuba2019}
{\sc Belle, V.} {\sc and} {\sc Juba, B.} 2019.
\newblock Implicitly learning to reason in first-order logic.
\newblock In {\em Advances in Neural Information Processing Systems 32}.
  3381--3391.

\bibitem[\protect\citeauthoryear{Belle and Lakemeyer}{Belle and
  Lakemeyer}{2014}]{bellejair}
{\sc Belle, V.} {\sc and} {\sc Lakemeyer, G.} 2014.
\newblock Multiagent only knowing in dynamic systems.
\newblock {\em The Journal of Artificial Intelligence Research (JAIR)\/}~{\em
  49}.

\bibitem[\protect\citeauthoryear{Belle and Lakemeyer}{Belle and
  Lakemeyer}{2015}]{Belle:2015ac}
{\sc Belle, V.} {\sc and} {\sc Lakemeyer, G.} 2015.
\newblock Only knowing meets common knowledge.
\newblock In {\em IJCAI}.

\bibitem[\protect\citeauthoryear{Bolander and Gierasimczuk}{Bolander and
  Gierasimczuk}{2015}]{Bolander2015}
{\sc Bolander, T.} {\sc and} {\sc Gierasimczuk, N.} 2015.
\newblock Learning actions models: Qualitative approach.
\newblock In {\em Logic, Rationality, and Interaction}, {W.~van~der Hoek},
  {W.~H. Holliday}, {and} {W.-f. Wang}, Eds. LNCS, vol. 9394. Springer, 40--52.

\bibitem[\protect\citeauthoryear{Daniely and Shalev-Shwartz}{Daniely and
  Shalev-Shwartz}{2016}]{daniely2016complexity}
{\sc Daniely, A.} {\sc and} {\sc Shalev-Shwartz, S.} 2016.
\newblock Complexity theoretic limitations on learning dnf’s.
\newblock In {\em COLT}.

\bibitem[\protect\citeauthoryear{de~C.~Ferreira, Fisher, and van~der
  Hoek}{de~C.~Ferreira et~al\mbox{.}}{2005}]{Ferreira2005}
{\sc de~C.~Ferreira, N.}, {\sc Fisher, M.}, {\sc and} {\sc van~der Hoek, W.}
  2005.
\newblock Logical implementation of uncertain agents.
\newblock In {\em 12th Portuguese Conference on Progress in Artificial
  Intelligence}.

\bibitem[\protect\citeauthoryear{De~Raedt and Kersting}{De~Raedt and
  Kersting}{2011}]{de2011statistical}
{\sc De~Raedt, L.} {\sc and} {\sc Kersting, K.} 2011.
\newblock Statistical relational learning.
\newblock In {\em Encyclopedia of Machine Learning}.

\bibitem[\protect\citeauthoryear{Esteban and Tor\'{a}n}{Esteban and
  Tor\'{a}n}{2001}]{esteban2001}
{\sc Esteban, J.~L.} {\sc and} {\sc Tor\'{a}n, J.} 2001.
\newblock Space bounds for resolution.
\newblock {\em Inf. Comput.\/}~{\em 171,\/}~1 (nov), 84–97.

\bibitem[\protect\citeauthoryear{Fabiano}{Fabiano}{2019}]{Fabiano2019MAEP}
{\sc Fabiano, F.} 2019.
\newblock Design of a solver for multi-agent epistemic planning.
\newblock {\em Electronic Proceedings in Theoretical Computer Science\/}~{\em
  306}, 403–412.

\bibitem[\protect\citeauthoryear{Fagin and Halpern}{Fagin and
  Halpern}{1994}]{174658}
{\sc Fagin, R.} {\sc and} {\sc Halpern, J.~Y.} 1994.
\newblock Reasoning about knowledge and probability.
\newblock {\em J. ACM\/}~{\em 41,\/}~2, 340--367.

\bibitem[\protect\citeauthoryear{Fagin, Halpern, Moses, and Vardi}{Fagin
  et~al\mbox{.}}{1995}]{reasoning:about:knowledge}
{\sc Fagin, R.}, {\sc Halpern, J.~Y.}, {\sc Moses, Y.}, {\sc and} {\sc Vardi,
  M.~Y.} 1995.
\newblock {\em Reasoning About Knowledge}.
\newblock {MIT} Press.

\bibitem[\protect\citeauthoryear{Fan, Wang, and van Ditmarsch}{Fan
  et~al\mbox{.}}{2013}]{FanWD13}
{\sc Fan, J.}, {\sc Wang, Y.}, {\sc and} {\sc van Ditmarsch, H.} 2013.
\newblock Knowing whether.
\newblock {\em CoRR\/}~{\em abs/1312.0144}.

\bibitem[\protect\citeauthoryear{Galil}{Galil}{1977}]{galil1977}
{\sc Galil, Z.} 1977.
\newblock On resolution with clauses of bounded size.
\newblock {\em SIAM Journal on Computing\/}~{\em 6,\/}~3, 444--459.

\bibitem[\protect\citeauthoryear{Getoor and Taskar}{Getoor and
  Taskar}{2007}]{getoor2007introduction}
{\sc Getoor, L.} {\sc and} {\sc Taskar, B.} 2007.
\newblock Introduction to statistical relational learning (adaptive computation
  and machine learning).

\bibitem[\protect\citeauthoryear{Halpern}{Halpern}{1997}]{halpern1997critical}
{\sc Halpern, J.} 1997.
\newblock {A Critical Reexamination of Default Logic, Autoepistemic Logic, and
  Only Knowing}.
\newblock {\em Computational Intelligence\/}~{\em 13,\/}~1, 144--163.

\bibitem[\protect\citeauthoryear{Halpern}{Halpern}{2003}]{halpern2003uncertainty}
{\sc Halpern, J.~Y.} 2003.
\newblock {\em Reasoning about Uncertainty}.
\newblock {MIT Press}.

\bibitem[\protect\citeauthoryear{Hoeffding}{Hoeffding}{1963}]{hoeffding1963probability}
{\sc Hoeffding, W.} 1963.
\newblock Probability inequalities for sums of bounded random variables.
\newblock {\em Journal of the American Statistical Association\/}~{\em
  58,\/}~301, 13--30.

\bibitem[\protect\citeauthoryear{Juba}{Juba}{2012}]{juba2012}
{\sc Juba, B.} 2012.
\newblock Learning implicitly in reasoning in pac-semantics.
\newblock {\em CoRR\/}~{\em abs/1209.0056}.

\bibitem[\protect\citeauthoryear{Juba}{Juba}{2013}]{juba2013}
{\sc Juba, B.} 2013.
\newblock Implicit learning of common sense for reasoning.
\newblock In {\em IJCAI}. 939--946.

\bibitem[\protect\citeauthoryear{Juba}{Juba}{2015}]{juba2015restricted}
{\sc Juba, B.} 2015.
\newblock Restricted distribution automatizability in {PAC}-semantics.
\newblock In {\em ITCS}. 93--102.

\bibitem[\protect\citeauthoryear{Kearns, Schapire, and Sellie}{Kearns
  et~al\mbox{.}}{1994}]{kearns1994toward}
{\sc Kearns, M.~J.}, {\sc Schapire, R.~E.}, {\sc and} {\sc Sellie, L.~M.} 1994.
\newblock Toward efficient agnostic learning.
\newblock {\em Machine Learning\/}~{\em 17,\/}~2-3, 115--141.

\bibitem[\protect\citeauthoryear{Khardon and Roth}{Khardon and
  Roth}{1997}]{KhardonRothJournal1997}
{\sc Khardon, R.} {\sc and} {\sc Roth, D.} 1997.
\newblock Learning to reason.
\newblock {\em J. ACM\/}~{\em 44,\/}~5, 697–725.

\bibitem[\protect\citeauthoryear{Kullmann}{Kullmann}{1999}]{kullmann99}
{\sc Kullmann, O.} 1999.
\newblock Investigating a general hierarchy of polynomially decidable classes
  of {CNF's} based on short tree-like resolution proofs.
\newblock Tech. Rep. TR99-041, ECCC.

\bibitem[\protect\citeauthoryear{Lakemeyer and Lesp{\'e}rance}{Lakemeyer and
  Lesp{\'e}rance}{2012}]{Lakemeyer2012EfficientRI}
{\sc Lakemeyer, G.} {\sc and} {\sc Lesp{\'e}rance, Y.} 2012.
\newblock Efficient reasoning in multiagent epistemic logics.
\newblock In {\em ECAI}.

\bibitem[\protect\citeauthoryear{Lakemeyer and Levesque}{Lakemeyer and
  Levesque}{2004}]{LakemeyerLevesque2004}
{\sc Lakemeyer, G.} {\sc and} {\sc Levesque, H.~J.} 2004.
\newblock Situations, {S}i! situation terms, {N}o!
\newblock In {\em KR}. 516--526.

\bibitem[\protect\citeauthoryear{Lakemeyer and Levesque}{Lakemeyer and
  Levesque}{2020}]{Lakemeyer2020}
{\sc Lakemeyer, G.} {\sc and} {\sc Levesque, H.~J.} 2020.
\newblock A first-order logic of limited belief based on possible worlds.
\newblock In {\em KR}. 624--635.

\bibitem[\protect\citeauthoryear{Levesque}{Levesque}{1990}]{LEVESQUE1990alliknow}
{\sc Levesque, H.~J.} 1990.
\newblock {All I know: A study in autoepistemic logic}.
\newblock {\em Artificial Intelligence\/}~{\em 42,\/}~2, 263--309.

\bibitem[\protect\citeauthoryear{Levesque and Lakemeyer}{Levesque and
  Lakemeyer}{2001}]{Levesque2001logic}
{\sc Levesque, H.~J.} {\sc and} {\sc Lakemeyer, G.} 2001.
\newblock {\em {The logic of knowledge bases}}.
\newblock The MIT Press.

\bibitem[\protect\citeauthoryear{Liu, Lakemeyer, and Levesque}{Liu
  et~al\mbox{.}}{2004}]{Liu2004ALO}
{\sc Liu, Y.}, {\sc Lakemeyer, G.}, {\sc and} {\sc Levesque, H.~J.} 2004.
\newblock A logic of limited belief for reasoning with disjunctive information.
\newblock In {\em KR}.

\bibitem[\protect\citeauthoryear{Liu and Levesque}{Liu and
  Levesque}{2003}]{liu2003tractability}
{\sc Liu, Y.} {\sc and} {\sc Levesque, H.} 2003.
\newblock A tractability result for reasoning with incomplete first-order
  knowledge bases.
\newblock In {\em IJCAI}. 83--88.

\bibitem[\protect\citeauthoryear{Liu and Levesque}{Liu and
  Levesque}{2005}]{liu2005tractable}
{\sc Liu, Y.} {\sc and} {\sc Levesque, H.} 2005.
\newblock {Tractable reasoning with incomplete first-order knowledge in dynamic
  systems with context-dependent actions}.
\newblock In {\em IJCAI}. 522--527.

\bibitem[\protect\citeauthoryear{Michael}{Michael}{2010}]{Loizos2010}
{\sc Michael, L.} 2010.
\newblock Partial observability and learnability.
\newblock {\em Artificial Intelligence\/}~{\em 174,\/}~11, 639–669.

\bibitem[\protect\citeauthoryear{Miller, Felli, Muise, Pearce, and
  Sonenberg}{Miller et~al\mbox{.}}{2016}]{miller2016}
{\sc Miller, T.}, {\sc Felli, P.}, {\sc Muise, C.}, {\sc Pearce, A.~R.}, {\sc
  and} {\sc Sonenberg, L.} 2016.
\newblock 'knowing whether' in proper epistemic knowledge bases.
\newblock In {\em Proceedings of the Thirtieth AAAI Conference on Artificial
  Intelligence}.

\bibitem[\protect\citeauthoryear{Mocanu, Belle, and Juba}{Mocanu
  et~al\mbox{.}}{2020}]{mocanu}
{\sc Mocanu, I.}, {\sc Belle, V.}, {\sc and} {\sc Juba, B.} 2020.
\newblock Polynomial-time implicit learnability in smt.
\newblock In {\em ECAI 2020}. Frontiers in Artificial Intelligence and
  Applications. IOS Press, Netherlands, 1152 -- 1158.

\bibitem[\protect\citeauthoryear{Muggleton and de~Raedt}{Muggleton and
  de~Raedt}{1994}]{Muggleton1994}
{\sc Muggleton, S.} {\sc and} {\sc de~Raedt, L.} 1994.
\newblock Inductive logic programming: Theory and methods.
\newblock {\em The Journal of Logic Programming\/}~{\em 19-20}, 629–679.
\newblock Special Issue: Ten Years of Logic Programming.

\bibitem[\protect\citeauthoryear{Muise, Belle, Felli, McIlraith, Miller,
  Pearce, and Sonenberg}{Muise et~al\mbox{.}}{2022}]{Muise2021MAEP}
{\sc Muise, C.}, {\sc Belle, V.}, {\sc Felli, P.}, {\sc McIlraith, S.}, {\sc
  Miller, T.}, {\sc Pearce, A.~R.}, {\sc and} {\sc Sonenberg, L.} 2022.
\newblock Efficient multi-agent epistemic planning: Teaching planners about
  nested belief.
\newblock {\em Artificial Intelligence\/}~{\em 302}, 103605.

\bibitem[\protect\citeauthoryear{Rader, Mocanu, Belle, and Juba}{Rader
  et~al\mbox{.}}{2021}]{rader2021}
{\sc Rader, A.}, {\sc Mocanu, I.~G.}, {\sc Belle, V.}, {\sc and} {\sc Juba, B.}
  2021.
\newblock Learning implicitly with noisy data in linear arithmetic.
\newblock In {\em IJCAI}.

\bibitem[\protect\citeauthoryear{Rosati}{Rosati}{2000}]{rosati2000}
{\sc Rosati, R.} 2000.
\newblock On the decidability and complexity of reasoning about only knowing.
\newblock {\em Artificial Intelligence\/}~{\em 116,\/}~1-2, 193--215.

\bibitem[\protect\citeauthoryear{Scherl and Levesque}{Scherl and
  Levesque}{2003}]{scherllevesque2003}
{\sc Scherl, R.~B.} {\sc and} {\sc Levesque, H.~J.} 2003.
\newblock Knowledge, action, and the frame problem.
\newblock {\em Artificial Intelligence\/}~{\em 144,\/}~1-2, 1--39.

\bibitem[\protect\citeauthoryear{Schwering and Pagnucco}{Schwering and
  Pagnucco}{2019}]{Schwering19}
{\sc Schwering, C.} {\sc and} {\sc Pagnucco, M.} 2019.
\newblock A representation theorem for reasoning in first-order multi-agent
  knowledge bases.
\newblock In {\em AAMAS}. 926--934.

\bibitem[\protect\citeauthoryear{Valiant}{Valiant}{1984}]{valiant84}
{\sc Valiant, L.} 1984.
\newblock A theory of the learnable.
\newblock {\em Communications of the ACM\/}~{\em 27,\/}~11, 1134--1142.

\bibitem[\protect\citeauthoryear{Valiant}{Valiant}{2000}]{Valiant2000}
{\sc Valiant, L.~G.} 2000.
\newblock {Robust logics}.
\newblock {\em Artificial Intelligence\/}~{\em 117,\/}~2, 231--253.

\bibitem[\protect\citeauthoryear{van Ditmarsch, Halpern, van~der Hoek, and
  Kooi}{van Ditmarsch et~al\mbox{.}}{2015}]{HansHalpern2015}
{\sc van Ditmarsch, H.}, {\sc Halpern, J.~Y.}, {\sc van~der Hoek, W.}, {\sc
  and} {\sc Kooi, B.~P.} 2015.
\newblock An introduction to logics of knowledge and belief.
\newblock {\em CoRR\/}~{\em abs/1503.00806}.

\end{thebibliography}

\label{lastpage}
\end{document}